\documentclass[10pt,twocolumn,letterpaper]{article}

\usepackage{iccv}
\usepackage{times}
\usepackage{epsfig}
\usepackage{graphicx}
\usepackage{subfigure}
\usepackage{amsmath}
\usepackage{amssymb}
\usepackage{algorithmic}
\usepackage{algorithm}
\usepackage{authblk}
\newtheorem{theorem}{Theorem}
\newtheorem{lemma}{Lemma}
\newtheorem{corollary}[theorem]{Corollary}
\newtheorem{definition}[theorem]{Definition}

\newtheorem{claim}{Claim}

\newcommand{\pr}[1]{\mathbb{P}\left\{{#1}\right\}}
\newcommand{\E}{\mathbb{E}}

\newfont{\msym}{msbm10}

\newcommand{\paren}[1]{\left({#1}\right)}
\newcommand{\brackets}[1]{\left[{#1}\right]}
\newcommand{\braces}[1]{\left\{{#1}\right\}}

\newcommand{\abs}[1]{\left\vert{#1}\right\vert}

\newcommand{\beq}[1]{\begin{equation}\label{#1}}
\newcommand{\eeq}{\end{equation}}
\newcommand{\beqa}{\begin{eqnarray}}
\newcommand{\eeqa}{\end{eqnarray}}

\newcommand{\thmref}[1]{Theorem~\ref{#1}}

\newcommand{\lemref}[1]{Lemma~\ref{#1}}

\newcommand{\corref}[1]{Corollary~\ref{#1}}

\newcommand{\mb}[1]{{\boldsymbol{#1}}}

\newcommand{\vx}{\mb{x}}

\newcommand{\vomega}{\mb{\omega}}

\newcommand{\vw}{\vomega}

\newcommand{\valpha}{\mb \alpha}

\newcommand{\newstufffroma}[1]{}

\newcommand{\newstufffrom}[1]{}

\newcommand{\oldnote}[2]{}

\newcommand{\comment}[1]{}
\newcommand{\commentout}[1]{}

\newcounter {mySubCounter}
\newcommand {\twocoleqn}[4]{
  \setcounter {mySubCounter}{0} %
  \let\OldTheEquation \theequation %
  \renewcommand {\theequation }{\OldTheEquation \alph {mySubCounter}}%
  \noindent \hfill%
  \begin{minipage}{.40\textwidth}
\vspace{-0.6cm}
    \begin{equation}\refstepcounter{mySubCounter}
      #1
    \end {equation}
  \end {minipage}
~~~~~~
  \addtocounter {equation}{ -1}%
  \begin{minipage}{.40\textwidth}
\vspace{-0.6cm}
    \begin{equation}\refstepcounter{mySubCounter}
      #3
    \end{equation}
  \end{minipage}%
  \let\theequation\OldTheEquation
}

\usepackage[pagebackref=true,breaklinks=true,letterpaper=true,colorlinks,bookmarks=false]{hyperref}

\iccvfinalcopy 

\def\iccvPaperID{3550} 

\ificcvfinal\fi
\begin{document}

\title{ASAP: Architecture Search, Anneal and Prune}
\author[1]{\small Asaf Noy}
\author[1]{\small Niv Nayman}
\author[1]{\small Tal Ridnik}
\author[1]{\small Nadav Zamir}
\author[2]{\small Sivan Doveh}
\author[1]{\small \\ Itamar Friedman}
\author[2]{\small Raja Giryes}
\author[1]{\small Lihi Zelnik-Manor}

\affil[1]{\footnotesize DAMO Academy, Alibaba Group (first.last@alibaba-inc.com)}
\affil[2]{\footnotesize School of Electrical Engineering, Tel-Aviv University Tel-Aviv, Israel}

\maketitle

\begin{abstract}

Automatic methods for Neural Architecture Search (NAS) have been shown to produce state-of-the-art network models. Yet, their main drawback is the computational complexity of the search process.
As some primal methods optimized over a discrete search space,  thousands of days of GPU were required for convergence.
A recent approach is based on constructing a differentiable search space that enables gradient-based optimization, which reduces the search time to a few days.
While successful, it still includes some noncontinuous steps, e.g., the pruning of many weak connections at once.
In this paper, we propose a differentiable search space that allows the annealing of architecture weights, while gradually pruning inferior operations. In this way, the search converges to a single output network in a continuous manner.
Experiments on several vision datasets demonstrate the effectiveness of our method with respect to the search cost and accuracy of the achieved model. Specifically, with $0.2$ GPU search days we achieve an error rate of $1.68\%$ on CIFAR-10. 
\end{abstract}
\vspace{-0.18cm}
\section{Introduction}
\label{sec:intro}
Over the last few years, deep neural networks highly succeed in computer vision tasks, mainly because of their automatic feature engineering. This success has led to large human efforts invested in finding good network architectures.
A recent alternative approach is to replace the manual design with an automated Network Architecture Search (NAS)~\cite{elsken2018neural}.
NAS methods have succeeded in finding more complex architectures that pushed forward the state-of-the-art in both image and sequential data tasks.

One of the main drawbacks of some NAS techniques is their large computational cost. 
For example, the search for NASNet~\cite{zophNasRL} and AmoebaNet~\cite{Real18Regularized}, which achieve state-of-the-art results in classification~\cite{huang2018gpipe}, 
have required $1800$ and $3150$ GPU days respectively. On a single GPU, this corresponds to years of training time.
More recent search methods, such as ENAS~\cite{ENAS} and DARTS~\cite{liu19darts}, reduced the search time to a few GPU days, while not compromising much on accuracy. 
While this is a major advancement, there is still a need to speed up the process to make the automatic search affordable and applicable to more problems. 

In this paper we propose an approach that further reduces the search time to a few hours, rather than days or weeks.
The key idea that enables this speed-up is relaxing the discrete search space to be continuous, differentiable, and annealable.
For continuity and differentiability, we follow the approach in DARTS~\cite{liu19darts}, which shows that a continuous and differentiable search space allows for gradient-based optimization, resulting in orders of magnitude fewer computations in comparison with black-box search, e.g., of~\cite{zophNasRL,Real18Regularized}.
We adopt this line of thinking, but in addition we construct the search space to be annealable, i.e., emphasizing strong connections during the search in a continuous manner.
We back this selection theoretically and demonstrate that an annealable search space implies a more continuous optimization, and hence both faster convergence and higher accuracy.

The annealable search space we define is a key factor to both reducing the network search time and obtaining high accuracy.
It allows gradual pruning of weak weights, which reduces the number of computed connections throughout the search, thus, gaining the computational speed-up. 
In addition, it allows the network weights to adjust to the expected final architecture, and choose the components that are most valuable, thus, improving the classification accuracy of the generated architecture. 
 
We validate the search efficiency and the performance of the generated architecture via experiments on multiple datasets.
The experiments show that networks built by our algorithm achieve either higher or on par accuracy to those designed by other NAS methods, even-though our search requires only a few hours of GPU rather than days or weeks. See for example results on CIFAR-10 in Fig.~\ref{fig:cifar_acc_days}. Specifically, searching for $4.8$ GPU-hours on CIFAR-$10$ produces a model that achieves a mean test error of $1.68\%$.

\begin{figure}[]    
          \centering
        \includegraphics[width=\linewidth]{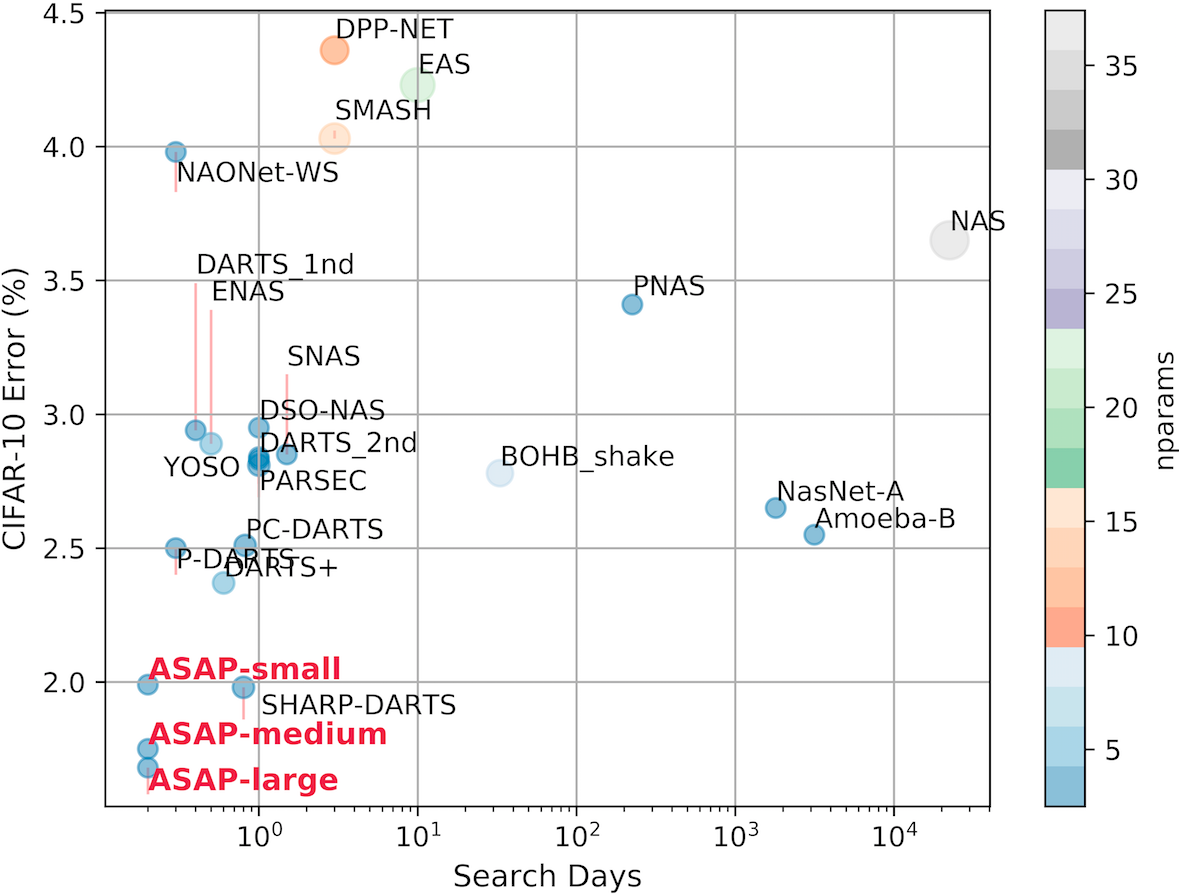}
        \caption{Comparison of CIFAR-10 error vs search days of top NAS methods. size and color are correlated with the memory footprint.}
        \label{fig:cifar_acc_days}
\end{figure}

\section{Related Work}
\label{sec:related_work}

{\bf Architecture search techniques.} 
A reinforcement learning based approach has been proposed by Zoph et al. for neural architecture search \cite{zophNasRL}. They use a recurrent network as a controller to generate the model description of a child neural network designated for a given task. The resulted architecture (NASnet) improved over the existing hand-crafted network models at its time. An alternative search technique has been proposed by Real et al. \cite{Real17Large, Real18Regularized}, where an evolutionary (genetic) algorithm has been used to find a neural architecture tailored for a given task. The evolved neural network \cite{Real18Regularized} (AmoebaNet), improved further the performance over NASnet.
Although these works achieved state-of-the-art results on various classification tasks, their main disadvantage is the amount of computational resources they demanded.

To overcome this problem, new efficient architecture search methods have been developed, showing that it is possible, with a minor loss of accuracy, to find neural architectures in just few days using only a single GPU. 
Two notable approaches in this direction are the differentiable architecture search (DARTS) \cite{liu19darts} and the efficient NAS (ENAS)~\cite{ENAS}. ENAS is similar to NASnet~\cite{zophNasRL}, yet its child models share weights of their shared operations from the large computational graph. During the search phase child models performance are estimated as a guiding signal to find more promising children. The weight sharing significantly reduces the search time since the child models performance are estimated with a minimal fine-tune training.

In DARTS, the entire computational graph, which includes repeating computation cells, is learned all together.
At the end of the search phase, a pruning is applied on connections and their associated operations within the cells, keeping those having the highest magnitude of their related architecture weights multipliers (which represents the connections' strength). These architecture weights are learned in a continuous way during the search phase. 
While DARTS achieves good accuracy, our hypothesis is that the harsh pruning which occurs only once at the end of the search phase is sub-optimal and that a gradual pruning of connections can improve both the search efficiency and accuracy.

SNAS \cite{snas} suggested an alternative method for selecting the connections in the cell. SNAS approaches the search problem as sampling from a distribution of architectures, where the distribution itself is learned in a continuous way. The distribution is expressed via slack softened one-hot variables that multiply the operations and imply their associated connection. SNAS introduces a temperature parameter that steadily decrease close to zero through time, forcing the distribution towards being degenerated, thus making slack variables closer to binary, which implies which connections should be sampled and eventually chosen.
Similarly to SNAS we suggest to perform an annealing process for connections and operations selection, but differently from SNAS we wish to learn the entire computational graph all together, while gradually annealing and pruning connections. 

In YOSO \cite{DSO}, they apply sparsity regularization over architecture weights and in the final stage they remove useless connections and isolated operations.

While both of these works \cite{snas,DSO} propose an alternative to the DARTS connections selection rule, none of them managed to surpass its accuracy on CIFAR-10. In our solution we improve both the accuracy and efficiency of the DARTS searching mechanism.

\vspace{7pt}
{\bf Neural networks pruning strategies.} 
Many pruning techniques exist for neural networks, since the networks may contain meaningful amount of operations which provides low gain \cite{Han2015Learning}. 
  
The methods for weight pruning can be divided into two main approaches: pruning a network post-training and performing the pruning during the training phase.

In {\em post-training pruning} the weight elimination is performed only after the network is trained to learn the importance of the connections in it \cite{LeCun90obg,Hassibi93obs,Han2015Learning}.
One selection criterion removes the weights based on their magnitude.
The main disadvantage of post-training methods is twofold: (i) long training time, which requires a train-prune-retrain schedule; (ii) pruning of weights in a single shot might lose some dependence between some of them that become more apparent if they are removed gradually.

{\em Pruning-during-training} as demonstrated in \cite{toPrune,comp} helps to reduce the overall training time and get a sparser network.
In \cite{toPrune}, a method for gradually pruning the network has been proposed, where the number of pruned weights at each iteration depends on the final desired sparsity and a given pruning frequency. Another work suggested a compression aware method that encourages the weight matrices to be of low-rank, which helps in  pruning them in a second stage \cite{comp}.
Another approach sets the coefficients that are smaller than a certain threshold to zero every K steps  \cite{TG}. Another pruning method uses simulated annealing\cite{SA97}. 


\section{Method}
\label{sec:method}

Our goal is to design an efficient algorithm for architecture search. 
We start by defining a differentiable search space that allows gradient-based optimization. A key characteristic of the search space we define is allowing for annealing of architecture weights. 

Annealing of the architecture weights and pruning weak ones is a key for converging to a single architecture. However, too fast annealing schedule or too strict pruning policy could end with convergence to inferior architectures.  \\
We provide a theory for selecting the critical combination of annealing schedule and pruning policy,
which guarantees that the pruning will not affect the quality of the final cell, as only inferior operations will be pruned along the search. Selecting a combination according to the theory will provide the advantages of annealing and pruning, converging to a better architecture, faster.

This theory provides us with some insights as to the importance of using an annealing schedule. Yet, as it requires a relatively slow schedule for the guarantees in it, we suggest afterwards another gradual pruning schedule that empirically leads to faster convergence in the network search.




\subsection{ASAP: Architecture Search, Anneal and Prune}

Our approach can be viewed as a generalization of DARTS~\cite{liu19darts} into an annealable search space.
The key idea behind DARTS is the definition of a continuous search space that facilitates the definition of a differentiable training objective.
DARTS continuous relaxation scheme leads to an efficient optimization and  a fast convergence in comparison to previous solutions. 

The architecture search in DARTS focus on finding a repeating structure in the network, which is called cell. They follow the observation that modern neural networks consist of one or few computational blocks (e.g. res-block) that are stacked together to form the final structure of the network. While a network might have hundreds of layers, the structure within each of them is repetitive. Thus, it is enough to learn one or two structures, which is denoted as cell, to design the whole network.

A cell is represented as a directed acyclic graph, consisting of an ordered sequence of nodes. Every node $\vx^{(i)}$ is a feature map and each directed connection $(i, j)$ is associated with some operation $o^{(i,j)}\in\mathcal{O}$ that transforms $\vx^{(j)}$ and connects it to $\vx^{(i)}$. Intermediate nodes are computed based on their predecessors,
\begin{equation} \label{eq:inter_nodes}
  \vx^{(i)} =\sum_{j<i}o^{(i,j)}\left(\vx^{(j)}\right)
\end{equation}

In particular, in DARTS, they search for two types of cells: normal and reduction. In normal cells the operations preserve the spatial dimensions, while in reduction cells the operations adjacent to the input nodes are of stride two, so the spatial dimensions at the output are halved.

The goal of the search phase is to select the operations $o^{(i,j)}$ which yield an overall architecture with the best performance. There are seven candidate operations: $3$x$3$ and $5$x$5$ separable and dilated separable
convolutions, $3$x$3$ max-pooling, $3$x$3$ average-pooling and
an identity.


This scheme does not encourage convergence towards a reasonably sized architecture, and produces an over-parametrized network. 
Therefore, hard pruning is applied over the network connections in order to derive the final \emph{child network}.
As shown in~\cite{snas}, this introduces a \emph{relaxation bias} that could result in a significant drop in accuracy between the un-pruned network and the final child network.



To overcome this limitation we define a search space that allows gradual pruning via annealing.
Our search converges gradually to the final child network without requiring a hard thresholding at the end of the search.

We construct an architecture by stacking normal and reduction cells, similarly to DARTS and ENAS. 
Also here, we connect nodes using the a mixed operation $\bar{o}^{(i,j)}$ edge. Yet, now we allow architecture weights annealing in it via a temperature parameter $T$:
\begin{eqnarray}\label{eq:mixedOpTemp}
\bar{o}^{(i,j)}(\vx;T) &=&\sum_{o\in\mathcal{O}}\Phi_o(\valpha^{(i,j)};T) \cdot o(\vx)
\end{eqnarray}
where $\valpha_{o}^{(i,j)}$ is the architecture weight associated with operation $o\in\mathcal{O}$ at edge $(i,j)$, and $\Phi_o$ forms a probability distribution.

The function $\Phi_o$ should be designed such that it guides the optimization to select a single operation out of the mixture in a finite time. Initially $\Phi_o$ should be a uniform distribution, allowing consideration of all the operations.
As the iterations continue the temperature is reduced and $\Phi_o$ should converge into a degenerated distribution that selects a single operation. 
Mathematically, this implies that $\Phi_o$ should be uniform for $T\rightarrow \infty$, and sparse for $T\rightarrow 0$.
Thus, we select the following probability distribution,
\begin{eqnarray}\label{eq:mixedOpTemp}
\Phi_o(\valpha^{(i,j)};T)=\frac{\exp{\braces{\frac{\valpha_o^{(i,j)}}{T}}} }{\sum_{o'\in\mathcal{O}}\exp{\braces{\frac{\valpha_{o'}^{(i,j)}}{T}}}}
\end{eqnarray}
This definition is closely related to the Gibbs-Boltzmann distribution, where the weights $\valpha_o^{(i,j)}$ correspond to negative energies, and operations $o^{(i,j)}$ to mixed system states~\cite{van1987simulated}. 

The architecture weights are updated via gradient descent,
\vspace{-0.32cm}
\begin{eqnarray} \label{GD_SM_update}
\alpha_{k}&\leftarrow& \alpha_{k}-\eta\nabla_{\alpha_{k}}\mathcal{L}_{val} \notag \\
\nabla_{\alpha_{k}}\mathcal{L}_{val}&=&\Phi_{o_k}(\valpha;T)\brackets{ \nabla_{\bar{o}}\mathcal{L}_{val}\cdot \left(o_{k}-\bar{o}\right)}
\end{eqnarray}
Note, that DARTS forms a special case of our approach when setting a fixed $T=1$. A key challenge here is how to select $T$. Hereafter in Section~\ref{sec:schedule}, we provide a theory for selecting it followed by a heuristic annealing schedule strategy to further speed up the convergence of our search. In the experiments we show its effectiveness.

\begin{figure}[]    
          \centering
        \includegraphics[width=\linewidth]{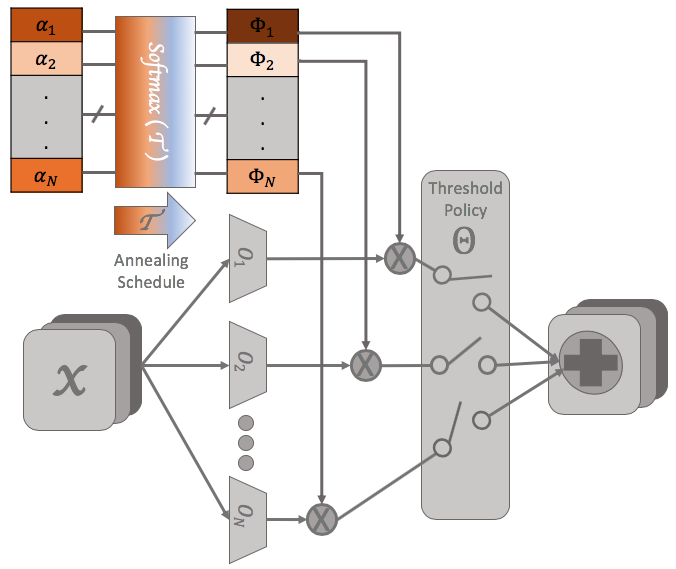}
        \caption{An illustration of our gradual annealing and pruning algorithm for architecture search depicted for a cell in the architecture. An annealing schedule is applied over architecture weights (top-left), which facilitates fast and continuous pruning of connections (right).}
        \label{fig:darts_annealing_schedule}
\end{figure}


We summarize our algorithm for Architecture Search, Anneal, and Prune (ASAP), outlined in Algorithm~\ref{alg:ASAP} and visualized in Figure~\ref{fig:darts_annealing_schedule}.
ASAP anneals and prunes the connections within the cell in a continuous manner, leading to a short search time and an improved cell structure. 
It starts with a few ``warm-up'' iterations for reaching well-balanced weights. Then it optimizes over the network weights while pruning weak connections, as described next. 

At initialization, network weights are randomly set. This could result in a discrepancy between parameterized operations (e.g. convolution layers) and non-parameterized operations (e.g. pooling layers). Applying pruning at this stage with imbalanced weights could lead to wrong premature decisions.
Therefore, we start with performing a number of gradient-descent based optimization steps over the network-weights only, without pruning, i.e. \emph{grace-cycles}. We denote the number of these cycles by $\tau$.

Once this warm-up ends, we perform a gradient-descent based optimization for the architecture-weights, in an alternating manner (with the updates of $\valpha$), after every network-weights optimization. 
Throughout this iterative process, the temperature $T$ decays according to a predefined annealing schedule. 
Operations are pruned if their corresponding architecture-weights go below a pruning threshold $\theta$, which is updated along the iterations.
The process ends once we meet a stopping criterion.
In our current implementation we used a simple criterion that stops when we reach a single operation per mixed operation.
Other criteria could be easily adopted as well.

\subsection{Annealing schedule and thresholding policy} \label{sec:schedule}

We now turn to describe the annealing schedule that determines the temperature $T$ through time and the threshold policy governing the updates of $\theta$.  
This choice is dominated by a trade-off. 
On the one hand, pruning-during-training simplifies the optimization objective and reduces the complexity, assisting the network to converge and reducing overfitting. 
In addition, it accelerates the search process, since pruned operations are removed from the optimization. 
This encourages the selection of fast temperature decay and a harsh thresholding. 
On the other hand, premature pruning of operations during the search could lead to a sub-optimal child network. 
This suggests we should choose a slow temperature decay and a soft thresholding.

To choose schedule and policy that balance this trade-off, we suggest Theorem~\ref{theorm:pac} that provides a functional form for $\paren{T_t,\theta_t}$. It views the pruning performed as a selection problem, i.e., pruning all the inferior operations, such that the ones with the highest expected architecture weight remains, in a setup where only the empirical value is measured. 
The theorem guarantees under some assumptions 
with a high probability, that Algorithm~\ref{alg:ASAP} prunes inferior operations out of the mixed-operation along the run (lines \ref{pruning_rule}-\ref{alg:prune}) and outputs the best operation (line \ref{alg:return})).
More formally, Algorithm~\ref{alg:ASAP} is a $\paren{0, \delta}$-PAC: 
\begin{definition}[($0, \delta$)-PAC]\label{def:pac}
    An algorithm that outputs the best operation with probability $1-\delta$ is ($0, \delta$)-PAC.
\end{definition} 

\begin{theorem} \label{theorm:pac}
    Assuming $\nabla_{\alpha_{t,i}}\mathcal{L}_{val}(\vw, \valpha; T_t)$ is independent through time and bounded by $\mathbf{L}$ in absolute value for all $t$ and $i$, then Algorithm \ref{alg:ASAP} with a threshold policy,
    \begin{align}
      \theta_t &= \nu_t e^{-t} 
      \\
      \nu_t&\in\Upsilon = \left\{\nu_t\mid \nu_t \geq 0; \lim_{t\rightarrow{}\infty}\frac{\log{(\nu_t)}}{t}=0\right\}
    \end{align}
    and an annealing schedule,
    \begin{align}
        T(t)&= \eta\mathbf{L} \rho_t \sqrt{\frac{8}{t} \log\paren{\frac{\pi^2 N t^2}{3\delta}}}
        \\
        \rho_t &= \frac{t}{t+\log\paren{\frac{1}{N\nu_t}}}
    \end{align} 
    is ($0,\delta$)-PAC.
\end{theorem}

In other words, the theorem shows that with a proper annealing and thresholding the algorithm has a high likelihood to prune only sub-optimal operations. A proof sketch appears below in Section~\ref{sec:theoretical_analysis}. The full proof is deferred to the supplementary material. 

While these theoretical results are strong, the critical schedule they suggest is rather slow, because PAC bounds tend to be overly pessimistic in practice. 
Therefore, for practical usage we suggest a more aggressive schedule.
It is based on observations explored in Simulated Annealing, which shares a similar tradeoff between the continuity of the optimization process and the time required for deriving the solution~\cite{busetti2003simulated}.

Among the many alternatives we choose the exponential schedule:
\begin{align}\label{eq:temp_schedule}
{T}(t)=T_0\beta^t,     
\end{align}
which has been shown to be effective when the optimization steps are expensive, e.g. \cite{ingber1989very},\cite{nourani1998comparison},\cite{kirkpatrick1983optimization}.
This schedule starts with a relatively high temperature $T_0$ and decays fast. 
As for the pruning threshold, we chose the simplest solution of a fixed value $\Theta\equiv\theta_0$. 
We demonstrate the effectiveness of our choices via experiments in Section~\ref{sec:exp}.

\begin{algorithm}[H] 
\caption{ASAP for a single Mixed Operation} 
\label{alg:ASAP}
\begin{algorithmic}[1]
\STATE \textbf{Input}: Operations $o_i\in\mathcal{O}~~i\in \{1,..,N\}$, 
\\  ~~~~~~~~~~~  Annealing schedule $T_t$,  \\    ~~~~~~~~~~~  Grace-temperature $\tau$,  \\ ~~~~~~~~~~~  Threshold policy $\theta_t$, 
\STATE \textbf{Init}: $\valpha_i \leftarrow 0,~~i\in \{1,..,N\}$.
\label{init}
\WHILE{ $|\mathcal{O}|>1$}
\STATE Update $\vw$ by descent step over $\nabla_{\vw}\mathcal{L}_{\mathrm{train}}(\vw,\valpha;T_t)$
\IF{$T_t<\tau$}
\STATE Update $\valpha$ by descent step over $\nabla_\valpha\mathcal{L}_{\mathrm{val}}(\vw,\valpha;T_t)$
\FOR{\textbf{each} $o_i \in \mathcal{O}$ such that ~$\Phi_{o_i}\paren{\valpha;T_t}<\theta_t$ \label{pruning_rule}}
\STATE $\mathcal{O}=\mathcal{O} \setminus \braces{o_i}$ \label{alg:prune}
\ENDFOR
\ENDIF 
\STATE Update $T_t$ 
\STATE Update $\theta_t$ 
\ENDWHILE
\RETURN {$\mathcal{O}$} \label{alg:return}
\end{algorithmic}
\end{algorithm}

\subsection{Theoretical Analysis}
\label{sec:theoretical_analysis}

The proof of the theorem relies on two main steps: (i) Reducing the algorithm to a \emph{Successive-Elimination} method \cite{even2006action}, and (ii) Bounding the probability of deviation of $\valpha$ from its expected value. We sketch the proof using Claim~\ref{claim:th_as_se} and Theorem~\ref{theorm:pac_per_op} below. 

\begin{claim}\label{claim:th_as_se}
    The pruning rule (Step \ref{pruning_rule}) in Algorithm~\ref{alg:ASAP} is equivalent to pruning $o_i\in\mathcal{O}$ at time t if,
    \begin{align}
        \frac{\alpha_{t,i}}{t} + \beta_t &<  \frac{\alpha_t^*}{t} -\beta_t,
    \end{align}
    where $\alpha_t^* = \max_i\left\{\alpha_{t,i}\right\}$ and $\beta_t = \frac{T_t}{2\rho_t}$. 
\end{claim}

Although involving the empirical values of $\alpha$, the condition in Claim~\ref{claim:th_as_se} avoids the pruning of the operation with the highest expected $\alpha$. For this purpose we bound the probability for the deviation of each empirical $\alpha$ from its expected value by the specified margin $\beta_t$.
We provide \thmref{theorm:pac_per_op}, which states that for any time t and operation $o_i\in \mathcal{O}$, the value of $\frac{\alpha_{t,i}}{t}$ is within $\beta_t$ of its expected value $\frac{\bar{\alpha}_{t,i}}{t}=\frac{1}{t}\sum_{s=1}^t\E\brackets{g_{s,i}}$,
where,
\begin{align}
    g_{t,i} = -\eta \nabla_{\alpha_{t,i}}\mathcal{L}_{val}(w, \alpha; T_t) \in [-\eta\mathbf{L}, \eta\mathbf{L}].
\end{align}

\begin{theorem}\label{theorm:pac_per_op}
    For any time t and operation $\{o_i\}_{i=1}^N \in \mathcal{O}$, we have,
    \begin{align}
       \pr{\frac{1}{t}\abs{\alpha_{t,i}-\bar{\alpha}_{t,i}} \leq \beta_t} \geq 1-\frac{\delta}{N}.
    \end{align}
\end{theorem}

Requiring this to hold for all the $N$ operations, we get that the probability of pruning the best operation is below $1-\delta$. 
Choosing an annealing schedule and threshold policy such that $\beta_t$ goes to zero as $t$ increases, guarantees that eventually all operations but the best one are pruned. This leads to the desired result. Full proofs appear in the supplementary material.

\section{Experiments}
\label{sec:exp}
To show the effectiveness of ASAP we test it on common benchmarks and compare to the state-of-the-art.
We experimented on the popular benchmarks, CIFAR-10 and ImageNet, as well as on five other classification datasets, for providing a more extensive evaluation.
In addition, we explore alternative pruning methods and compare them to ASAP.
\subsection{Architecture search on CIFAR-10}
\label{arch_search_cifar}
We search on CIFAR-10 for convolutional cells in a small parent network. 
Then we build a larger network by stacking the learned cells, train it on CIFAR-10 and compare the results against common NAS methods.

We create the parent network by stacking $8$ cells with architecture weights sharing.
Each cell contains $4$ ordered nodes, each of which is connected via mixed operations to all previous nodes in the cell and also to the two previous cells' outputs.
The output of a cell is a concatenation of the outputs of all the nodes within the cell.

The search phase lasts until each mixed operation is fully pruned, or until we reach $50$ epochs. We use the first-order approximation \cite{liu19darts}, relating to \( \valpha \) and \( \vw \) as independent thus can be optimized separately. The train set is divided into two parts of equal sizes: one is used for training the operations' weights $\vw$ and the other for training the architecture weights $\valpha$.

For the gradual annealing, we use Equation~\ref{eq:temp_schedule} with an initial temperature of $T_0=1.3$ and an epoch decay factor of $\beta=0.95$, thus $T\approx 0.1$ at the end of the search phase. Considering we have $N$ possible operations to choose from in a mixed operation, we set our pruning threshold to $\theta_t\equiv\frac{0.4}{N}$. Other hyper-parameters follow \cite{liu19darts}.

As the search progresses, continuous pruning reduces the network size. With a batch size of $96$, one epoch takes $5.8$ minutes in average on a single GPU\footnote{Experiments were performed using a NVIDIA GTX 1080Ti GPU.}, summing up to $4.8$ hours in total for a single search. 




Figure \ref{fig:search_duration} illustrates an example of the epoch duration and the network sparsity ($1$ minus the relative part of the pruned connections) during a search. 
The smooth removal of connections is evident. This translates to a significant decrease in an epoch duration along the search. 
Note that the first five epochs are grace-cycles, where only the network weighs are trained as no operations are pruned.
\begin{figure}
  \begin{center}
    \includegraphics[width=\linewidth]{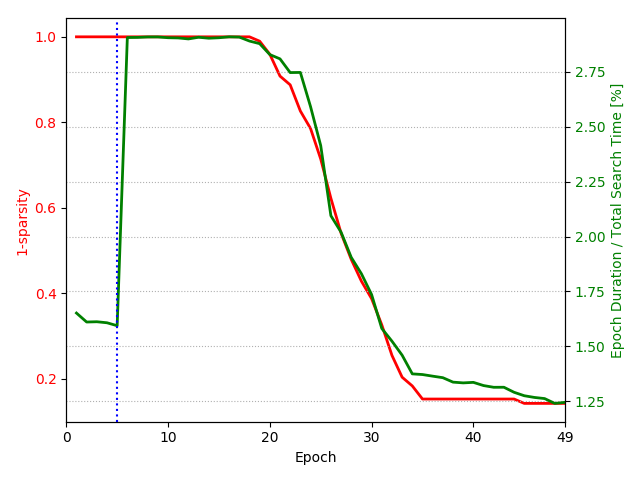}
    \caption{Relative epoch duration during search and model $1-$sparsity vs epoch number. The dotted line represents grace cycles.\label{fig:search_duration}}
    \end{center}
\end{figure}
Figures~\ref{fig:normal cell} and~\ref{fig:reduction cell} show our learned normal and reduction cells, respectively.
\begin{figure}[h!]
  \begin{subfigure}{}
  \begin{center}
    \includegraphics[width=\linewidth]{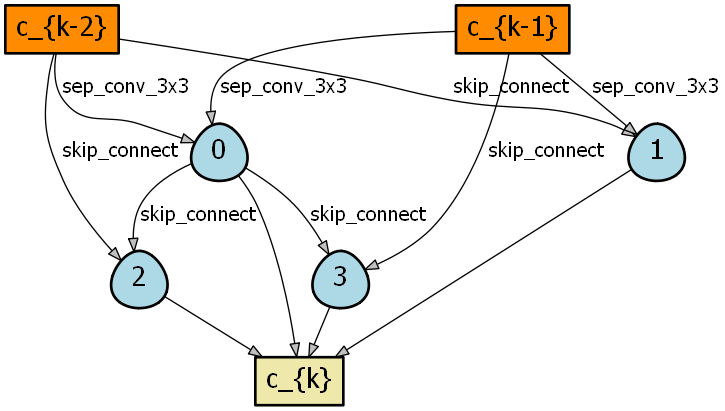}
    \caption{ASAP learned normal cell on CIFAR-10.\label{fig:normal cell}}
    \end{center}
  \end{subfigure}
  \begin{subfigure}{}
  \begin{center}
    \includegraphics[width=\linewidth]{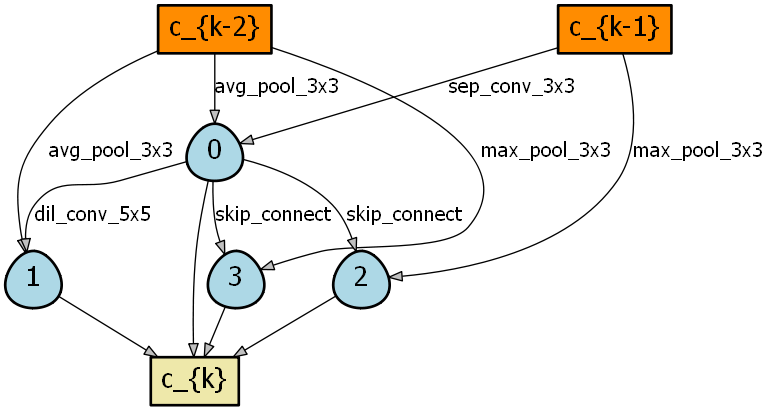}
    \caption{ASAP learned reduction cell on CIFAR-10.\label{fig:reduction cell}}
    \end{center}
  \end{subfigure}
\end{figure}

\subsection{CIFAR-10 Evaluation Results}\label{CIFAR-10 Evaluation Results}
We built the evaluation network by stacking $20$ cells, $18$ normal cells and $2$ reduction cells. We place the reduction cells after $1/3$ and $2/3$ of the network, where after each reduction we double the amount of channels in the network. 
We trained the network for $1500$ epochs using a batch size of $128$ and SGD optimizer with nesterov-momentum.
Our learning rate regime was composed of $5$ cycles of power cosine annealing learning rate \cite{hundt2019sharpdarts}, with amplitude decay factor of 0.5 per cycle.
For regularization we used cutout \cite{devries2017improved}, auxiliary towers \cite{szegedy2015going}, scheduled drop-path \cite{larsson2016fractalnet},  label smoothing, AutoAugment \cite{AutoAug} and weight decay. 
To understand the effect of the network size on the final accuracy, we chose to test $3$ architecture configurations with $36$, $44$ and $50$ initial network channels, which we named respectively ASAP-Small, ASAP-Medium and ASAP-Large. 

Table~\ref{cifar-10-results} shows the performance of our learned models compared to other state-of-the-art NAS methods. Figure \ref{fig:cifar_acc_days} provides an even more comprehensive comparison in a graphical manner.
\begin{table}
    \begin{center}
        \begin{tabular}{l|c|c|c}
            Architecture
            & \multicolumn{1}{|p{1.4cm}}{\centering Test error\\(\%)}
            & \multicolumn{1}{|p{1.05cm}}{\centering Params\\(M) }
            & \multicolumn{1}{|p{1.0cm}}{\centering Search\\cost $\downarrow$}\\
            \hline
            AmoebaNet-A \cite{Real18Regularized}     & 3.34 $\pm$ 0.06       & 3.2  & 3150 \\
            AmoebaNet-B \cite{Real18Regularized} & 2.55 $\pm$ 0.05       & 2.8  & 3150 \\
            NASNet-A \cite{NASNET}  & 2.65      & 3.3  & 1800 \\
            PNAS \cite{PNAS} & 3.41  & 3.2  & 150 \\
            \hline
            SNAS \cite{snas} & 2.85 $\pm$ 0.02 & 2.8 & 1.5 \\
                        DSO-NAS \cite{DSO} & 2.95 $\pm$ 0.12 & 3 & 1 \\
           PARSEC \cite{casale2019probabilistic} & 2.81 $\pm$ 0.03 & 3.7 & 1 \\          
            DARTS(2nd) \cite{liu19darts}      & 2.76 $\pm$ 0.06        & 3.4  & 1 \\           
            PC-DARTS DL2 \cite{laube2019prune}      & $2.51\pm 0.09$ & 4.0 & 0.82 \\
            
            DARTS+ \cite{liang2019darts+}      & $2.37\pm 0.13$ & 4.3 & 0.6 \\
            
            ENAS \cite{ENAS} & 2.89  & 4.6  & 0.5 \\
            DARTS(1nd) \cite{liu19darts}      & 2.94       & 2.9  & 0.4 \\
            
            P-DARTS \cite{chen2019progressive} & 2.50  & 3.4  & 0.3 \\
            DARTS(1nd) \cite{liu19darts}      & 2.94       & 2.9  & 0.4 \\

            NAONet-WS \cite{NAO} & 3.53 & \bf{2.5} & 0.3 \\    
            \hline
            ASAP-Small & 1.99 & \bf{2.5} & \bf{0.2} \\
            ASAP-Medium & 1.75 & 3.7 & \bf{0.2} \\
            ASAP-Large & \bf{1.68} & 6.0 & \bf{0.2} \\
        \end{tabular}
    \end{center}
    \caption{Classification errors of ASAP compared to state-of-the-art NAS methods on CIFAR-10. The Search cost is measured in GPU days.}
    \label{cifar-10-results}
\end{table}

Table \ref{cifar-10-results} and Figure \ref{fig:cifar_acc_days} show that our ASAP based network outperforms previous state-of-the-art NAS methods, both in terms of the classification error and the search time.
\subsection{Transferability Evaluation}
Using the cell found by ASAP search on CIFAR-10, we preformed transferability tests on $6$ popular classification benchmarks: ImageNet, CINIC10, Freiburg, CIFAR-100, SVHN and Fashion-MNIST.

\textbf{ImageNet} \label{imagenet_paragraph}
For testing ASAP cell transferability performance on larger datasets, we experimented on the popular ImageNet dataset \cite{imagenet_cvpr09}. 
The network was composed of $14$ stacked cells, with two initial stem cells. We used 50 initial channels, so the total number of network FLOPs is below 600[M], similar to other ImageNet architectures with small computation regime\cite{liu19darts}.
We trained the network for $250$ epochs using a nesterov-momentum optimizer.
The results are presented in Table~\ref{datasets-tranfer-results}.
\begin{table*}
    \begin{center}
        \begin{tabular}{l|c|c|c|c|c|c|c}
    
            Architecture 
            & \multicolumn{1}{|p{1.45cm}}{\centering CINIC-10\\Error(\%) }
            & \multicolumn{1}{|p{1.65cm}}{\centering FREIBURG\\Error(\%) }
            & \multicolumn{1}{|p{1.7cm}}{\centering CIFAR-100\\Error(\%) }
            & \multicolumn{1}{|p{1.2cm}}{\centering SVHN\\Error(\%) }
            & \multicolumn{1}{|p{1.4cm}}{\centering FMNIST\\Error(\%) }
            & \multicolumn{1}{|p{1.4cm}}{\centering ImageNet\\Error(\%) }
            & \multicolumn{1}{|p{1.4cm}}{\centering Search\\cost }\\
            \hline
            Known SotA 
            & 8.6 \cite{cinic10} 
            & 21.1 \cite{Freiburg}
            & 8.7 \cite{huang2018gpipe} 
            & 1.02 \cite{AutoAug} 
            & 3.65 \cite{zhong2017random}
            & 15.7 \cite{huang2018gpipe} 
            & - \\
            \hline
            AmoebaNet-A \cite{Real18Regularized}  & 7.18 & 11.8 & 15.9 & 1.93 & 3.8 & \textbf{24.3} & 3150\\
            NASNet \cite{NASNET} & 6.93 & 13.4 & 15.8 & 1.96 & 3.71 & 26.0 & 1800\\

            PNAS \cite{PNAS}  & 7.03 & 12.3 & 15.9 & 1.83 & 3.72 & 25.8 & 150\\   
            SNAS \cite{snas}  & 7.13 & 14.7 & 16.5 & 1.98 & 3.73 & 27.3 & 1.5\\                 
            DARTS-Rev1 \cite{liu19darts}  & 7.05 & 11.4 &  15.8 & 1.94 & 3.74 & 26.9 & 1\\ 
            DARTS-Rev2 \cite{liu19darts}  
            & 6.88
            & 10.8 
            & 15.7
            & 1.85
            & \textbf{3.68} 
            & 26.7 
            & 1\\
       
            \hline
            ASAP 
            & \textbf{6.83}
            & \textbf{10.7}
            & \textbf{15.6}
            & \textbf{1.81} 
            & 3.73 & 24.4 & \bf{0.2}\\
            \end{tabular}
    \end{center}
    \caption{Transferability classification error of ASAP, compared to top NAS cells, on several vision classification datasets. Error refers to top-1 test error. Search cost is measured in GPU days. DARTS Rev1 and Rev2 refers to the best 2nd order cells published in revision 1 and 2 of \cite{liu19darts} respectively.}
    \label{datasets-tranfer-results}
\end{table*}
It can be seen that the cell found by ASAP transfers well to ImageNet - accuracy second only to \cite{Real18Regularized}, with significant faster search time.

\textbf{Other classification datasets}
We further extend our transferability testing by training our ASAP cell and other top published NAS cells on $5$ additional benchmarks: CINIC10,  Freiburg, CIFAR-100, SVHN and Fashion-MNIST. Datasets details appear in \ref{train_details}. For a fair comparison, we trained all of the cells using the publicly available DARTS training code \footnote{https://github.com/quark0/darts}, with exactly the same network configuration and hyperparameters, except from the cell itself. Results are presented in Table \ref{datasets-tranfer-results}.

Table \ref{datasets-tranfer-results} shows that our ASAP cell transfers well to other computer vision datasets, surpassing all other NAS cells in four out of five datasets tested, while having the lowest search cost. In addition, note that on two datasets, CINIC-10 and FREIBURG, our ASAP network accuracy is better than previously known state-of-the-art.
\subsection{Other Pruning Methods} 
\label{sec:other}
In addition to comparisons to other SotA algorithms, we also evaluate alternative pruning-during-training techniques. 
We select two well known pruning approaches and adapt those to the neural architecture search framework.
The first approach is magnitude based. 
It naively cuts connections with the smallest weights, as is the practice in some network pruning strategies, e.g. \cite{Han2015Learning, toPrune} (as well as in \cite{liu19darts}, yet with a harsh single pruning-after-training step).
The second approach \cite{accum} prunes connections with the lowest accumulated absolute value of the gradients of the loss with respect to $\valpha$. 
Our evaluation is based on the number of operations to prune in each step, according to the formula suggested in~\cite{toPrune}:
\begin{eqnarray}
\label{eq:sparsity_schedule}
 s_t = s_f + (s_i-s_f)\paren{1-\frac{t-t_0}{n\Delta t}}^p 
 \end{eqnarray} 
 where $p\in \braces{1,3}$, $t\in\braces{t_0, t_0+\Delta t,...,t_0+n\Delta t}$, $s_f$ is the final desired sparsity (the fraction of pruned weights), $s_i$ is the initial sparsity, which is $0$ in our case, $t_0$ is the iteration we start the pruning at, and $E_f = t_0+n\Delta t$ is the total number of iterations used for the search. For $p=1$ the pruning rate is constant, while for $p=3$ it decreases. 
The evaluation was performed on CIFAR-10 with $t_0=20$ and $E_f$ in the range from $50$ to $70$. Table \ref{ablation table} presents the results. 
\begin{figure*}[t]
\begin{center}
      \begin{minipage}{0.465\textwidth}
        \begin{center}
           \includegraphics[width=\linewidth]{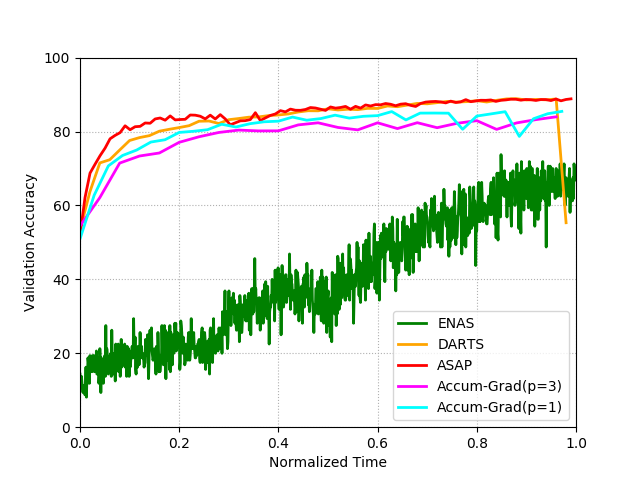}
           \caption{Search validation accuracy for ENAS, DARTS, ASAP and Accum-Grad over time (scaled to [0, 1]).}
        \label{fig:accuracies}
        \end{center}
      \end{minipage}
      \begin{minipage}{0.05\textwidth}
          \includegraphics[width=\linewidth]{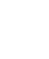}
      \end{minipage}
      \begin{minipage}{0.45\textwidth}
        \begin{center}
           \includegraphics[width=\linewidth]{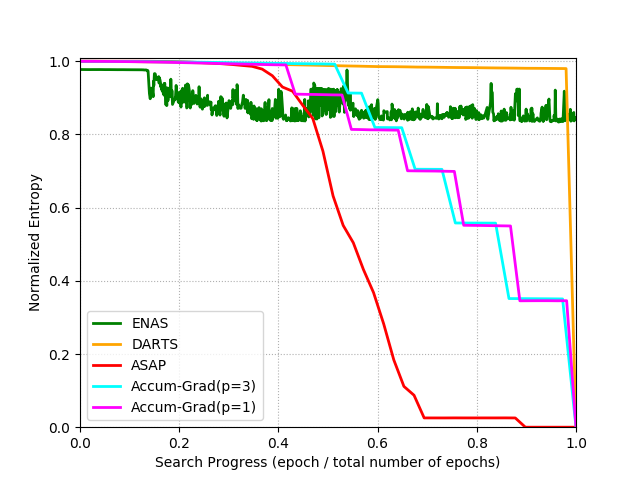}
           \caption{The total normalized entropy during search of normal and reduction cells for ENAS,
           DARTS, ASAP and Accum-Grad vs the time scaled to [0, 1].}
           \label{fig:entropies}
        \end{center}
      \end{minipage}
  \end{center}
\end{figure*}

We can see from Table \ref{ablation table} that both pruning methods achieved lower accuracy than ASAP on CIFAR-10, with a larger memory footprint.

\subsection{Search Process Analysis }

In this section we wish to provide further insights as to why ASAP leads to a higher accuracy with a lower search time. 
To do that we explore two properties along the iterative learning process: the validation accuracy and the cell entropy.
We compare ASAP to two other efficient methods: DARTS~\cite{liu19darts}, and the reinforcement-learning based ENAS~\cite{ENAS}.
We further compare to the pruning alternative based on accumulated gradients described in Section \ref{sec:other}, as it achieved better results on CIFAR-$10$ them magnitude pruning, as shown in Table \ref{ablation table}.
We compare with both $p\in \braces{1,3}$.

Figure \ref{fig:accuracies} presents the validation accuracy along the architecture search iterations.
ENAS achieves low and noisy accuracy along the search, inferior to all differentiable-space based methods. 
DARTS accuracy climbs nicely across iterations, but then significantly drops at the end of the search due to the relaxation bias following the harsh prune.
The pruning-during-training methods suffer from perturbations 
in accuracy following pruning, however, ASAP's perturbations are frequent and smaller, as its architecture-weights which get pruned are insignificant. Its validation accuracy quickly
recovers due to the continuous annealing of weights.
Therefore, With reduced network complexity, it achieves the highest accuracy at the end.
\begin{table}
    \begin{center}
        \begin{tabular}{l|c|c}

            Architecture type 
            & \multicolumn{1}{|p{2.0cm}|}{\centering Test error \\ Top-$1$(\%) }
            & \multicolumn{1}{|p{2.0cm}}{\centering Params \\ (M) } \\
            
                        \hline
            DARTS(2st) &  2.76 & 3.4  \\  
            \hline
            Magnitude prune & 2.9 & 4.4  \\   
            Accum-Grad prune & 2.76 & 4.3 \\ 
            \hline
            ASAP-Small & 1.99 & 2.5  \\
      \end{tabular}
    \end{center}
    \caption{Comparison of pruning methods on CIFAR-10.} 
    \label{ablation table}
\end{table}
Figure~\ref{fig:entropies} illustrates the average cell entropy over time, averaged across all the mixed operations within the cell, normalized by $\ln{\paren{|\mathcal{O}|}}$. 
While the entropy of DARTS and ENAS remain high along the entire search, the entropy of the pruning methods decrease to zero as those do not suffer from a relaxation bias. 
The high entropy of DARTS when final prune is done leads to a quite arbitrary selection of a child model, as top architecture weights are of comparable value.
It can be seen that the entropy of ASAP decreases gradually, hence allowing efficient optimization in the differentiable-annealable space. This provides further motivation to the advantages of ASAP.

\section{Conclusion}
\label{sec:conc}
In this paper we presented ASAP, an efficient state-of-the-art neural architecture search framework. 
The key contribution of ASAP is the annealing and gradual pruning of connections during the search process.
This approach avoids hard pruning of connections that is used by other methods.
The gradual pruning decreases the architecture search time, while improving the final accuracy due to smoother removal of connections during the search phase. 

On CIFAR-$10$, our ASAP cell outperforms other published cells, both in terms of search cost and in terms of test error. 
We also demonstrate the effectiveness of ASAP cell by showing good transferability quality compared to other top NAS cells on multiple computer vision datasets. 


{\small
\bibliographystyle{ieee}
\bibliography{references}
}

\newpage
\clearpage
\section{Supplementary Material}
\subsection{Proof of Theorem \ref{theorm:pac}}
\label{pac_bounds_in_details}
Let us set some terms,
\begin{align}
    g_{t,i} = -\eta \nabla_{\alpha_{t,i}}\mathcal{L}_{val}(w, \alpha; T_t)
\end{align}
Then we have $|g_{t,i}| \leq \eta\cdot\mathbf{L}$ and,
\begin{align}
   \alpha_{t,i} &= \alpha_{0,i} - \eta \sum_{s=0}^t \nabla_{\alpha_{s,i}}\mathcal{L}_{val}(w, \alpha; T_t)
   = \sum_{s=0}^t g_{s,i}
\end{align}
As we initialize $\alpha_{0,i}=0$ for all $i=1,\dots,N$ in \ref{init}. \\
Then we prove claim \ref{claim:th_as_se},
\begin{proof}\label{proof:threshold_as_successive_elimination}
    The pruning rule can be put as following,
    \begin{align}
        &\Phi_{o_i}(\alpha_{t}; T_t) < \theta_t \\
        &\frac{e^{\frac{\alpha_{t,i}}{T_t}}}{\sum_{j=1}^N e^{\frac{\alpha_{t,i}}{T_t}}} < \theta_t \\
        \alpha_{t,i} &< T_t\log\paren{\sum_{j=1}^N e^{\frac{\alpha_{t,i}}{T_t}}} -T_t\log\paren{\frac{1}{\theta_t}} \\
        &\leq T_t\paren{\frac{\alpha_t^*}{T_t} + \log(N)} -T_t\log\paren{\frac{1}{\theta_t}} \label{bounding_log_sum_exp}\\
        &< \alpha_t^* - T_t\paren{\log\paren{\frac{1}{\theta_t}} - \log(N)}\\
        &< \alpha_t^* -T_t\paren{ t+\log\paren{\frac{1}{N\nu_t}}} \label{theta_to_nu}\\
        &< \alpha_t^* - t T_t\rho_t^{-1} \\
        &<  \alpha_t^* -2t\beta_t
    \end{align}
    Where \ref{theta_to_nu} is by setting $\theta_t=\nu_t e^{-t}$ and \ref{bounding_log_sum_exp} is since,
    \begin{align} \label{lse}
        \log\paren{\sum_{i=1}^N e^{x_{i}}} \leq \max_i x_i + \log(N)
    \end{align}
    Finally we have,
    \begin{align}
        \frac{\alpha_{t,i}}{t} + \beta_t &<  \frac{\alpha_t^*}{t} -\beta_t,        
    \end{align}
\end{proof}

We wish to bound the probability of pruning the best operation, i.e. the operation with the highest expected architecture parameter $\alpha$.
Although involving the empirical values of $\alpha$, we show that the condition in claim \ref{claim:th_as_se} avoids the pruning of the operation with the highest expected $\alpha$. For this purpose we bound the probability for the deviation of each empirical $\alpha$ from its expected value by the specified margin $\beta_t$. 
For this purpose we introduce the following concentration bound,
\begin{lemma}[Hoeffding \cite{hoeffding1963probability}]\label{lemma:hoeffding}
    Let $g_1,\dots, g_t$ be independent bounded
    random variables with $g_s \in [a_s,b_s]$, where -$\infty < a_s \leq b_s < \infty$ for all $s=1,\dots,t$. Then,
    \begin{align}
       &\pr{\frac{1}{t}\sum_{s=1}^t (g_s - \E\brackets{g_s}) \geq \,\,\,\,\beta} \leq e^{-\frac{2\beta^2 t^2}{\sum_{s=1}^t \paren{b_s-a_s}^2}}
       \\
       &\pr{\frac{1}{t}\sum_{s=1}^t (g_s - \E\brackets{g_s}) \leq -\beta} \leq e^{-\frac{2\beta^2 t^2}{\sum_{s=1}^t \paren{b_s-a_s}^2}}
    \end{align}
\end{lemma}

Our main argument, described in \thmref{theorm:pac_per_op}, is that at any time t and for any operation $o_i$, 
the empirical value $\frac{\alpha_{t,i}}{t}$ is within $\beta_t$ of its expected value $\frac{\bar{\alpha}_{t,i}}{t}=\frac{1}{t}\sum_{s=1}^t\E\brackets{g_{s,i}}$. For the purpose of proving \thmref{theorm:pac_per_op}, we first prove the following \corref{cor:bound_for_t},
\begin{corollary}\label{cor:bound_for_t}
    For any time t and operations $\{o_i\}_{i=1}^N \in \mathcal{O}$ we have,
    \begin{align}
       \pr{\frac{1}{t}\abs{\alpha_{t,i}-\bar{\alpha}_{t,i}} > \beta_t} \leq 
       \frac{6\delta}{\pi^2 N t^2}
    \end{align}
\end{corollary}

\begin{proof}
    \begin{align}
       &\pr{\frac{1}{t}\abs{\alpha_{t,i}-\bar{\alpha}_{t,i}} > \beta_t} 
       \\&=
       \pr{\frac{\alpha_{t,i}}{t}-\frac{\bar{\alpha}_{t,i}}{t} > \beta_t \bigcup \frac{\alpha_{t,i}}{t}-\frac{\bar{\alpha}_{t,i}}{t} < -\beta_t} 
       \\ &\leq \label{hoeffding_union_bound}
       \pr{\frac{\alpha_{t,i}}{t}-\frac{\bar{\alpha}_{t,i}}{t} > \beta_t} +  \pr{\frac{\alpha_{t,i}}{t}-\frac{\bar{\alpha}_{t,i}}{t} < -\beta_t} 
       \\ &\leq  \notag
       \pr{\frac{1}{t}\sum_{s=1}^t (g_{s,i} - \E\brackets{g_{s,i}}) \geq \beta_t} 
       \\&+\pr{\frac{1}{t}\sum_{s=1}^t (g_{s,i}- \E\brackets{g_{s,i}}) \leq -\beta_t} 
       \\ &\leq \label{double_hoeffding}
       2e^{-\frac{2\beta_t^2 t}{4\eta^2\mathbf{L}^2}}
       \\&\leq \label{setting_beta_t}
       \frac{6\delta}{\pi^2 N t^2} 
    \end{align}
    Where, \ref{hoeffding_union_bound} is by the union bound, \ref{double_hoeffding} is by \lemref{lemma:hoeffding} with $g_{t,i}\in[-\eta\mathbf{L},\eta\mathbf{L}]$ and \ref{setting_beta_t} is by setting,
    \begin{align} 
        \beta_t=\eta\mathbf{L} \sqrt{\frac{2}{t}log\paren{\frac{\pi^2 N t^2}{3\delta}}}
    \end{align}
\end{proof}

We can now prove \thmref{theorm:pac_per_op},
\begin{proof}
    \begin{align}
       &\pr{\bigcap_{t=1}^\infty\frac{1}{t}\abs{\alpha_{t,i}-\bar{\alpha}_{t,i}} \leq \beta_t} 
       \\&=  
       1-\pr{\bigcup_{t=1}^\infty\frac{1}{t}\abs{\alpha_{t,i}-\bar{\alpha}_{t,i}} > \beta_t} 
       \\&\geq \label{union_bound_over_t}
       1-\sum_{t=1}^\infty\pr{\frac{1}{t}\abs{\alpha_{t,i}-\bar{\alpha}_{t,i}} > \beta_t} 
       \\&\geq \label{hoeffding_bound_over_t}
       1-\frac{6\delta}{\pi^2 N}\sum_{t=1}^\infty \frac{1}{t^2}
       \\&= 
       1 - \frac{\delta}{N}
    \end{align}
    where \eqref{union_bound_over_t} is by the union bound, \eqref{hoeffding_bound_over_t} is by \corref{cor:bound_for_t}.
\end{proof}
\bigskip

Requiring \thmref{theorm:pac_per_op} to hold for all of the operations, we get by the union bound that, the probability of pruning the best operation is less than $\delta$. 
Furthermore, since $\nu_t \in\Upsilon$, $\rho_t$ goes to 1 as t increases and $T_t$ goes to zero together with $\beta_t$. Thus eventually all operations but the best one are pruned. This completes our proof.


\subsection{ASAP search and train details}
In order to conduct a fair comparison, we follow \cite{liu19darts} for all the search and train details. This excludes the new annealing parameters and those related to the training of additional datasets, which are not mentioned in \cite{liu19darts}. Our code will be made available for future public use.

\subsubsection{Search details}
{\bf Data pre-processing.} We apply the following:
\begin{itemize}
    \item Centrally padding the training images to a size of $40$x$40$.
    \item Randomly cropping back to the size of $32$x$32$.
    \item Randomly flipping the training images horizontally.
    \item Standardizing the train and validation sets to be of a zero-mean and a unit variance.
\end{itemize}
{\bf Operations and cells.} We select from the operations mentioned in \ref{arch_search_cifar}, with a stride of $1$ for all of the connections within a cell but for the reduction cells' connections to the previous cells, which are with a stride of $2$. 
Convolutional layers are padded so that the spatial resolution is kept. The operations are applied in the order of ReLU-Conv-BN. Following \cite{zophNasRL},\cite{Real18Regularized}, depthwise separable convolutions are always applied twice. 
The cell's output is a $1$x$1$ convolutional layer applied on all of the cells' four intermediate nodes' outputs concatenated, such that the number of channels is preserved.
In CIFAR-$10$, the search lasts up to $0.2$ days on NVIDIA GTX 1080Ti GPU. \\
{\bf The annealing schedule.}  
We use the exponential decay annealing schedule as described in Algorithm \ref{alg:ASAP} when setting the annealing schedule as in \ref{eq:temp_schedule} with $\paren{T_0,\beta,\tau}=\paren{1.6,0.95,1}$. It was selected for obtaining a final temperature of $0.1$ and $5$ epochs of grace-cycles.\\
{\bf Alternate optimization.} 
For a fair comparison, we use the exact training settings from \cite{liu19darts}. We use the Adam optimizer \cite{kingma2014adam} for the architecture weights optimization with momentum parameters of $\paren{0.5,0.999}$ and a fixed learning rate of $10^{-3}$. For the network weights optimization, we use SGD \cite{robbins1951stochastic} with a momentum of $0.9$ as the learning rate is following a cosine annealing schedule \cite{loshchilov2016sgdr} with an initial value of $0.025$.

\subsubsection{Training details}
\label{train_details}
{\bf CIFAR-10.} \label{cifar10_training}
The training architecture consists of $20$ cells stacking up: $18$ normal cells and $2$ reduction cells, located at the $1/3$ and $2/3$ of the total network depth respectively.
We double the number of channels after each reduction cell.
We train the network for $1500$ epochs with a batch size of $128$. We use the SGD nesterov-momentum optimizer \cite{robbins1951stochastic} with a momentum of $0.9$, following a cycles cosine annealing learning rate \cite{loshchilov2016sgdr} with an initial value of $0.025$. We apply a weight decay of $3\cdot10^{-4}$ and a norm gradient clipping at $5$. We add an auxiliary loss \cite{szegedy2015going} after the last reduction cell with a weight of $0.4$.
In addition to data pre-processing, we use cutout augmentations \cite{devries2017improved} with a length of $16$ and a drop-path regularization \cite{larsson2016fractalnet} with a probability of $0.2$. 
WE also use the AutoAugment augmentation scheme. \\
{\bf ImageNet.}
Our training architecture starts with stem cells that reduce the input image resolution from $224$ to $56$ ($3$ reductions), similar to MobileNet-V1 \cite{mobilenet-v1}. We then stack $14$ cells: $12$ normal cells and $2$ reduction cells. The reduction cells are placed after the fourth and eighth normal cells. The normal cells start with $50$ channels, as the number of channels is doubled after each reduction cell. We also added SE layer\cite{hu2018squeeze} at the end of each cell. In total, the network contains $5.7$ million parameters.
We train the network on $2$ GPUs for $250$ epochs with a batch size of $256$. 
We use the SGD nesterov-momentum optimizer \cite{robbins1951stochastic} with a momentum of $0.9$, following a cosine learning rate with an initial learning rate value of $0.2$.
We apply a weight decay of $1\cdot10^{-4}$ and a norm gradient clipping at $5$. We add an auxiliary loss after the last reduction cell with the weight of $0.4$ and a label smoothing \cite{reed2014training} of $0.1$.
During training, we normalize the input image and crop it with a random cropping factor in the range of $0.08$ to $1$. In addition, autoaugment augmentations and randomly horizontal flipping are applied. During testing, we resize the input image to the size of $256$x$256$ and applying a fixed central crop to the size of $224$x$224$. \\
{\bf Additional datasets.}
Our additional classification datasets consist of the Following:

{\bf {CINIC-10:} \cite{cinic10}}
 is an extension of CIFAR-$10$ by ImageNet images, down-sampled to match the image size of CIFAR-$10$. 
 It has $270,000$ images of $10$ classes, i.e. it has larger train and test sets than those of CIFAR-$10$. 

{\bf {CIFAR-100:} \cite{cifar100}}
A natural image classification dataset, containing $100$ classes with $600$ images per class. The image size is $32$x$32$ and the train-test split is $50,000$:$10,000$ images respectively. 

{\bf {FREIBURG:} \cite{Freiburg}}
A groceries classification dataset consisting of $5000$ images of size $256$x$256$, divided into $25$ categories.
It has imbalanced class sizes ranging from $97$ to $370$ images per class. Images were taken in various aspect ratios and padded to squares.

{\bf {SVHN:} \cite{SVHN}}
A dataset containing real-world images of digits and numbers in natural scenes. It consists of $600,000$ images of size $32$x$32$, divided into $10$ classes. The dataset can be thought of as a real-world alternative to MNIST, with an order of magnitude more images and significantly harder real-world scenarios. 

{\bf {FMNIST:} \cite{fashionMnist}}
A clothes classification dataset with a $60,000$:$10,000$ train-test split. Each example is a grayscale image of size $28$x$28$, associated with a label from $10$ classes of clothes. It is intended to serve as a direct drop-in replacement for the original MNIST dataset as a benchmark for machine learning algorithms.

The training scheme use for those was similar to the one used for CIFAR-10, with some minor adjustments and modifications - mainly the use  of standard color augmentations instead of autoaugment regimes, and default training length of $600$ epochs instead of $1500$ epochs.
For the FREIBURG dataset, we resized the original images from $256$x$256$ to $64$x$64$. For CINIC-10, we were training for $400$ epochs instead of $600$, since this dataset is quite large.
Note that for each of those datasets, all of the cells were trained with exactly the same network architecture and hyper-parameters, unlike our ImageNet comparison at \ref{imagenet_paragraph}, where each cell was embedded into a different architecture and trained with a completely different scheme.

\end{document}